\documentclass{article}
\PassOptionsToPackage{table}{xcolor}
\usepackage{microtype}
\usepackage{graphicx}
\usepackage{subfigure}
\usepackage{multirow}
\usepackage{caption}
\usepackage{bm}
\usepackage{booktabs} %

\usepackage{hyperref}

\PassOptionsToPackage{numbers, compress}{natbib}
\usepackage[main, preprint]{neurips_2026}

\usepackage{amsmath}
\usepackage{amssymb}
\usepackage{mathtools}
\usepackage{amsthm}

\usepackage[capitalize,noabbrev]{cleveref}

\theoremstyle{plain}
\newtheorem{theorem}{Theorem}[section]

\newtheorem{corollary}[theorem]{Corollary}
\theoremstyle{definition}

\theoremstyle{remark}

\usepackage[textsize=tiny]{todonotes}

\usepackage[utf8]{inputenc} %
\usepackage[T1]{fontenc}    %
\usepackage{hyperref}       %
\hypersetup{hidelinks}
\usepackage{url}            %
\usepackage{booktabs}       %
\usepackage{amsfonts}       %
\usepackage{nicefrac}       %
\usepackage{microtype}      %
\usepackage[table]{xcolor}
\title{Constrained Group Relative Policy Optimization}

\author{%
  Roger Girgis\thanks{Equal contribution. Primary contact: \texttt{roger.girgis@gmail.com}} \\
  Mila -- Quebec AI Institute\\
  \'{E}cole Polytechnique de Montr\'{e}al\\
  \And
  Rodrigue de Schaetzen\footnotemark[1] \\
  Mila -- Quebec AI Institute\\
  Universit\'{e} de Montr\'{e}al\\
  \And
  Luke Rowe \\
  Mila -- Quebec AI Institute\\
  Universit\'{e} de Montr\'{e}al\\
  \And
  Azal\'{e}e Robitaille \\
  Mila -- Quebec AI Institute\\
  Universit\'{e} de Montr\'{e}al\\
  \And
  Christopher Pal \\
  Mila -- Quebec AI Institute\\
  \'{E}cole Polytechnique de Montr\'{e}al\\
  Universit\'{e} de Montr\'{e}al, CIFAR AI Chair\\
  \And
  Liam Paull \\
  Mila -- Quebec AI Institute\\
  Universit\'{e} de Montr\'{e}al, CIFAR AI Chair\\
}

\begin{document}

\maketitle

\begin{abstract}
Group Relative Policy Optimization (GRPO) remains the dominant critic-free approach for fine-tuning LLMs and VLMs, but its compatibility with constrained policy optimization (e.g. for safety-critical domains) has not been carefully examined. In this work, we introduce Constrained GRPO, a Lagrangian-based extension of GRPO for constrained policy optimization. We show that the standard practice of scalarizing rewards before normalization introduces a critical Lagrangian-specific failure mode:  GRPO’s within-group normalization makes constrained optimization highly sensitive to how multi-component learning signals are aggregated.  We show that scalarizing rewards before normalization introduces shared-denominator coupling, so that changing one multiplier alters not only the emphasis on its corresponding constraint, but also the relative weighting of the reward and other constraints.  We address this with a simple but crucial modification: scalarizing standardized advantages rather than rewards. This yields a better-conditioned update by addressing the coupling induced by reward scalarization, resulting in better-behaved multiplier dynamics and more stable constraint enforcement in practice. Empirically, across a controlled gridworld, a real-world autonomous driving benchmark, and a mathematical reasoning task, Constrained GRPO consistently achieves better adherence to specified constraints while maintaining or improving task performance.
\end{abstract}

\section{Introduction}
\label{sec:introduction}

Recent progress in reinforcement learning (RL) has shown that large multimodal models can be optimized with outcome-based signals. 
This has enabled the finetuning of models on mathematical reasoning tasks \cite{cobbe2021training,Wei2022ChainOT,Lewkowycz2022SolvingQR}, embodied decision making~\cite{zhang2025safevla}, among other domains where success can be evaluated with programmatic or learned rewards.
Group Relative Policy Optimization (GRPO)~\cite{Shao2024DeepSeekMathPT} removes the reliance on a critic for advantage estimation by a within-group normalization of the rewards obtained over multiple sampled outputs for the same query.
The critic-free structure of GRPO is appealing when using a value function during training is challenging / computationally expensive.

Existing GRPO-based methods have largely approached multi-objective optimization through manually weighting the importance of different reward terms~\cite{Li2025OptimizingSA, ichihara2025mo, liu2026gdpo}.
In contrast, the constrained policy optimization literature offers a principled framework, namely Constrained Markov Decision Processes (CMDPs), for enforcing behavioural preferences and constraints by explicitly defining thresholds on behaviours rather than employing fixed weights.
However, GRPO has not yet been integrated into a constrained optimization framework.
In this work, we propose \emph{Constrained GRPO}, a constrained policy optimization framework employing GRPO as the policy optimization algorithm.
Within our framework, we specify desirable behaviours with indicator cost functions~\cite{Roy2021DirectBS} and enforce target behaviour rates (in percentage) using a Lagrangian relaxation with learned multipliers.

A central challenge arises from the interaction between Lagrangian methods and GRPO's advantage estimation through within-group standardization.
A simple approach is to first \emph{scalarize rewards}, i.e., combine reward and cost terms into a single scalar value using the current multipliers, and then apply the within-group standardization.
However, we show that this seemingly innocuous choice introduces an additional source of coupling in the constrained update: 
because the scalarized signal is normalized by a shared, multiplier-dependent denominator, changing one multiplier can affect not only its corresponding objective term, but also the relative contribution of of other components to the policy update.
We provide a formal analysis of the gradient estimator sensitivity to changes in the multipliers, as well as an intuitive framing of how multiplier distortion occurs (e.g. Figure~\ref{fig:intro_fig}).
Motivated by this analysis, Constrained GRPO employs \emph{scalarizing advantages}, i.e., standardize each component within group first, and then combine the standardized components using the Lagrangian multipliers.
This design yields better-conditioning and well-behaved multiplier dynamics by removing the shared-denominator coupling that is induced by reward scalarization.
Importantly, this does not imply exact preservation of the original unnormalized Lagrangian objective; rather, it produces a more stable learning signal in the normalized component basis, which we validate empirically.

Prior and concurrent work has also found advantage scalarization beneficial in the realm of multi-objective GRPO~\cite{Li2025OptimizingSA, ichihara2025mo}; our contribution places this observation in the context of constrained optimization, showing why it is essential for effective Lagrangian constraint enforcement under GRPO.
Our main contributions are:
(1) We introduce Constrained GRPO, a constrained policy optimization framework built on GRPO, using indicator cost functions, Lagrangian relaxation and scalarizing advantages to effectively enforce behaviour rates without manually tuning reward weights.
(2) We analyze the common approach of scalarizing rewards before normalizing, and formally characterize the local response of the primal update to changes in the dual variables in Theorem~\ref{thm:scalarized_reward_logit_sensitivity}.
(3) We study the alternative approach of scalarizing advantages and empirically demonstrate improved stability and faithfulness of the behaviour specification in a gridworld environment, a realistic simulation-based autonomous driving setting, and a mathematical reasoning task.

\begin{figure*}[t]
\centering
\includegraphics[width=\textwidth]{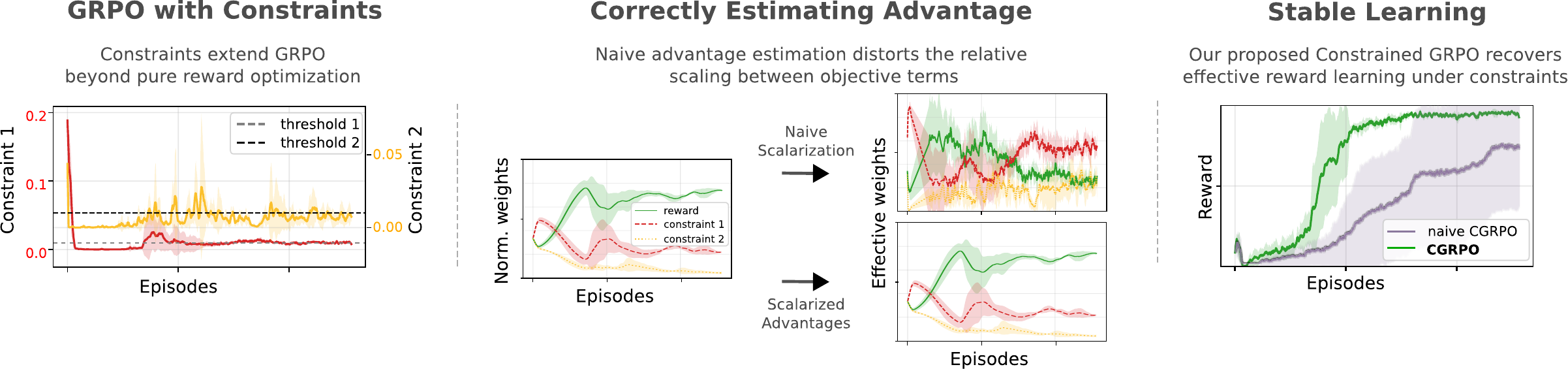}
\captionsetup{font=footnotesize}
\caption{Overview of our proposed approach. We extend GRPO to constrained policy optimization by enforcing user-specified behavioral constraints with thresholds (left). Naively scalarizing reward distorts reward–constraint trade-offs, producing inconsistent effective weights (middle; e.g., red/green swap). Scalarized advantages maintains the relative multiplier weighting between advantage components, yielding more stable learning and effective constraint enforcement (right).}
\label{fig:intro_fig}
\vspace{-5mm}
\end{figure*}

\section{Related Work}
\label{sec:related_work}

\textbf{Reward Specification.}
A central motivation for our work is the difficulty of specifying complex behaviours through a single reward function.  
This challenge has inspired several alternative approaches in the literature.
Imitation Learning~\cite{zheng2021imitation} bypasses reward engineering entirely by using expert demonstrations to define the task. 
Other methods introduce human feedback during training, either to guide the agent toward desired behaviours~\cite{christiano2017deep} or to prevent catastrophic errors during exploration~\cite{saunders2017trial}. 
Our approach shares a similar motivation—capturing designer intent more directly—but does so by specifying constraints via indicator functions and thresholds, avoiding the need for demonstrations or continuous human intervention.
Another promising direction is the use of natural language to specify tasks~\cite{goyal2019using,macglashan2015grounding}. 
While this can provide a more intuitive interface for human designers, language is inherently ambiguous and can be exploited by agents (reward hacking). 
Moreover, these methods typically require training an additional language-to-reward model, adding further complexity.
Another interesting approach is Inverse Reward Design~\cite{hadfield2017inverse,mindermann2018active,ratner2018simplifying}, which treats the provided reward function as a single observation of the designer's latent intent and learns a probabilistic model over possible true reward functions. 
This is particularly useful in settings with changing environments, but our work focuses on fixed-distribution environments where explicit behaviour specification is more appropriate.

\textbf{Multi-Objective Optimization with Large Models.}
In the context of large multimodal (vision-)~language models, there is extensive work on safety and multi-objective optimization via reinforcement learning~\cite{Moskovitz2023ConfrontingRM,Zhang2025SafeVLATS,Dai2023SafeRS,Ji2025SafeRS, Bai2022ConstitutionalAH}. 
For example, Constitutional AI~\cite{Bai2022ConstitutionalAH} can be viewed as balancing multiple objectives—most notably helpfulness and harmlessness—yet in practice these objectives are typically reduced to a single scalar reward and optimized with standard RLHF ~\cite{ouyang2022traininglanguagemodelsfollow} pipelines, for instance PPO-style updates~\cite{schulman2017proximal}. 
\citet{Moskovitz2023ConfrontingRM} frame composite reward models in RLHF as a constrained optimization problem, enforcing thresholds on individual reward components via a CMDP formulation and using dynamic weight adjustment (i.e., Lagrange multipliers) to mitigate over-optimization. 
Similarly, ~\citet{Dai2023SafeRS} formulate alignment as maximizing reward subject to cost constraints in RLHF, explicitly trading off helpfulness against safety-related costs. 
Finally, ~\citet{Wang2024InterpretablePV} study multi-objective reward models in RLHF, emphasizing interpretability as a means to help designers verify adherence to human preferences.

While several recent methods target safety alignment for RL-finetuned (vision-)language models, they are largely developed outside the GRPO setting. 
More recent work~\citep{Li2025OptimizingSA} train a multi-head reward model that predicts $K{=}4$ attributes (politeness, meaningfulness, actionability, safety) and then form a single scalar reward via a fixed weighted sum (using $w_k{=}1$ for all components), before optimizing the resulting objective with GRPO. \citet{ichihara2025mo} propose Multi-Objective GRPO (MO-GRPO) to reduce reward-hacking behavior under multiple objectives, and advocate scalarized advantages to equalize the influence of components with different variances. 
Concurrent to our work, GDPO~\cite{liu2026gdpo} also supports advantage scalarization and emphasizes the difficulty of selecting reward weights that reflect desired priorities, proposing conditional objectives that activate certain terms only once prerequisite scores exceed a minimum threshold.
In contrast, we study advantage scalarization through the lens of \emph{constrained} policy optimization. 
Our goal is to ensure that increases in the Lagrangian multipliers in Constrained GRPO translate into stronger enforcement of the specified behaviors. 
As shown experimentally, scalarizing advantages provides this property, whereas scalarizing rewards can implicitly reweight objectives and undermine effective Lagrangian optimization.

\section{Background}
\label{sec:background}

\textbf{Lagrangian methods.}
Constrained Markov Decision Processes (CMDPs) are commonly addressed with Lagrangian methods~\cite{achiam2017constrained,ray2019benchmarking,stooke2020responsive,zhang2020first,zhang2025safevla}, which cast the problem as a saddle-point optimization~\cite{uzawa,polyak,korpelevich1976}. 
These methods introduce $K$ nonnegative multipliers $\bm{\lambda}=\{\lambda_k\}_{k=1}^K$, where each $\lambda_k\ge 0$ controls the penalty for the $k$-th constraint $C_k$.
Letting $J_R(\pi)$ and $J_{C_k}(\pi)$ denote the expected returns of the reward $R$ and the constraint $C_k$ under policy $\pi$, with constraint threshold $d_k$, the resulting min--max problem is
\stepcounter{equation}
\begin{align}
\max_{\pi}\ \min_{\bm{\lambda}\ge 0}\ \mathcal{L}(\pi,\bm{\lambda}),
\tag{\theequation}\label{eq:cmdp-lagrangian}
\quad\text{where} \quad
\mathcal{L}(\pi,\bm{\lambda})
&= J_R(\pi)-\sum_{k=1}^K \lambda_k\big(J_{C_k}(\pi)-d_k\big). \notag
\end{align}
A standard approach applies gradient ascent in $\pi$ and gradient descent in $\bm{\lambda}$~\cite{gda}, which often yields good feasible solutions on CMDP benchmarks~\cite{achiam2017constrained,ray2019benchmarking,stooke2020responsive,zhang2020first}. 

Prior work~\cite{Roy2021DirectBS} employed multiplier normalization:
\begin{equation}
\label{eq:lambda_softmax_full}
\lambda_k \;=\; \frac{\exp(z_k)}{\exp(z_R) + \sum_{k'=1}^{K}\exp(z_{k'})},\;\; k=1,\dots,K.
\end{equation}
which includes a dummy variable $z_R$ for the main reward function such that $\lambda_R := 1-\sum_{k=1}^{K}\lambda_k$ and $z_k$'s are the \emph{unnormalized} multipliers (logits).
In this work, we adopt this mechanism since it ensures that one constraint can dominate the policy updates if necessary while maintaining stability~\cite{Roy2021DirectBS}.
The policy update (i.e., primal update) now corresponds to optimizing a linear combination of the reward and constraint signals,
\begin{equation}
\label{eq:combined_rewards}
S^{\bm{\lambda}} = \lambda_R R - \sum_{k=1}^K \lambda_k C_{k},
\end{equation}
and can be carried out with any policy optimization method (e.g., PPO~\cite{schulman2017proximal}, SAC~\cite{haarnoja2018soft}). 
The multiplier update (i.e., dual update) follows
\begin{align}
\label{eq:lambda_update}
\nabla_{\lambda_k}\mathcal{L}(\pi,\bm{\lambda})=-(J_{C_k}(\pi)-d_k),
\end{align}
typically implemented with a learning rate and a projection onto $\lambda_k\ge 0$ using max-clipping. If constraint $k$ is violated ($J_{C_k}>d_k$), $\lambda_k$ increases, strengthening its penalty; if it is satisfied ($J_{C_k}<d_k$), $\lambda_k$ decreases, allowing optimization to focus on return and other active constraints.

\textbf{Group-Relative Policy Optimization (GRPO).}
Reinforcement learning has proven effective for finetuning large models, complementing supervised finetuning (SFT): 
while SFT supports instruction and format adherence, RL finetuning often improves generalization, particularly in out-of-distribution regimes~\cite{Chu2025SFTMR}. 
Early work used PPO with learned reward models for preference optimization~\cite{ouyang2022traininglanguagemodelsfollow}. 
Assuming access to a queryable scalar reward (e.g., BLEU/ROUGE, verifier-based math checking), \citet{Shao2024DeepSeekMathPT} proposed eliminating the value model by estimating advantages within a group of samples drawn from the same prompt. 
Given a start state $s_0$, GRPO samples $G$ candidate trajectories $\{ \bm{\tau}_i = (s_0, a_0^i, \dots, s_{T-1}^i, a_{T-1}^i, s_T^i)\}_{i=1}^{G}$, from the behavior policy $\bm{\tau}_i \sim \pi_{\theta_{\mathrm{old}}}(\cdot \mid s_0)$, and computes group-relative advantages by standardizing returns:
\begin{equation*}
A_{i}=\frac{R(\bm{\tau}_i)-\mathrm{mean}(\mathbf{r})}{\mathrm{std}(\mathbf{r})}, \quad \mathbf{r} = [R(\bm{\tau}_1),\dots,R(\bm{\tau}_G)].
\end{equation*}
Using $A_{i}$, the PPO objective is modified into GRPO:
\begin{equation}
\begin{aligned}
\mathcal{J}_{\text{GRPO}}(\theta) = \mathbb{E}_{\{\bm{\tau}_i\}_{i=1}^G \sim
\pi_{\theta_{\text{old}}}}
     \left[ \frac{1}{G}\sum_{i=1}^G\sum_{t=0}^{T-1} \{ \min [b_1 \cdot A_{i}, b_2 \cdot A_{i}] - \beta \mathbb{D}_{\text{KL}} [\pi_{\theta} \mid\mid \pi_{\text{ref}}] \} \right ] , \notag\\
\text{where} \quad
b_1 = \frac{\pi_{\theta}(a_{t}^i\mid s_{t}^i)}{\pi_{\theta_{\text{old}}}(a_{t}^i\mid s_{t}^i)} , \quad
b_2 = \text{clip}\left ( \frac{\pi_{\theta}(a_{t}^i\mid s_{t}^i)}{\pi_{\theta_{\text{old}}}(a_{t}^i \mid s_{t}^i)}, 1 + \epsilon, 1 - \epsilon \right ),
\end{aligned}
\end{equation}

where the KL term (scaled by hyperparameter $\beta \geq 0$) directly regularizes the trained policy toward a fixed reference policy $\pi_{\mathrm{ref}}$, rather than being folded into the reward.
GRPO has shown strong performance in mathematical reasoning~\cite{Shao2024DeepSeekMathPT} and has been adopted in various robotics work~\cite{Li2025DriveR1BR,chen2025riftgrouprelativerlfinetuning,Pei2025AdvancingMT, wang2025alpamayo}.

\section{Constrained GRPO}
\label{sec:constrained_grpo}

We study constrained policy optimization with GRPO as the policy-gradient algorithm and Lagrangian relaxation as the constraint-handling mechanism. Our analysis examines how the choice of advantage estimation in GRPO under multiple reward and constraint signals affects the primal-dual dynamics, and identifies the conditions under which the dual updates reliably translate into primal behavior. We assume one reward and $K$ behavioral constraints; the extension to multiple rewards is straightforward.

In large-model finetuning, a common practice is to manually set the weights $\{\lambda_k\}$  and keep them fixed throughout training.
We model each desired behavior $k$ as a constraint and enforce it through a cost function $C_k$ with a target threshold $d_k$ and normalized multiplier $\lambda_k$ (Eq.~\ref{eq:lambda_softmax_full}).
A key design choice is how to express these thresholds in an intuitive and behavior-agnostic way.
Following prior work on direct behavior specification~\cite{Roy2021DirectBS}, we restrict our attention to indicator cost functions $C_k(s,a) = \mathbb{I}\big[\text{constraint $k$ violated in }(s,a)\big]$,
for which the threshold $\tilde d_k$ directly corresponds to a desired behavior rate.
In this work, we focus on trajectory-level constraints and therefore drop the explicit dependence on $(s,a)$.\footnote{The analysis in this work easily extends to state-action-level constraints.}
We estimate constraint violations from a batch of experience by the empirical average of the cost $C_k$.
The Lagrangian parameters are optimized using the update rule in Eq.~\eqref{eq:lambda_update}, with $d_k \rightarrow \tilde d_k$ and $J_{C_k}(\pi)$ computed as the average over $\{C_k\}_{i=1}^N$.

\textbf{GRPO Advantage Estimation.}
GRPO estimates advantages by within-group standardization: subtracting the group mean and dividing by the group standard deviation. 
There are two natural ways to compute the advantage.
The first scalarizes at the reward level by linearly combining rewards and costs with the Lagrange multipliers before estimating the advantage.
The second scalarizes at the advantage level by first computing component-wise advantages for each reward and cost and then linearly combining them with the Lagrange multipliers.
In what follows, we analyze the differences between these two approaches using a sensitivity analysis on the unclipped policy-gradient, revealing why scalarized advantages offers a more practical approach from the perspective of primal learning dynamics.
Note that this sensitivity analysis is taken with respect to the logits \(z_k\) (see Eq. \ref{eq:lambda_softmax_full}) rather than treating the multipliers \(\lambda_k\) as independent coordinates.
This is a consequence of the multiplier normalization that we employ and we provide the corresponding sensitivity analysis for the unnormalized case in Appendix \ref{thm:scalarized_reward_logit_sensitivity_apdx} (up to Eq.~\ref{eq:nosoftnorm_proof_screw}).

\paragraph{Scalarized rewards.}
The most direct approach first forms a scalarized signal $S^{\bm{\lambda}}$ as done in \eqref{eq:combined_rewards},
and then applies GRPO normalization:
\begin{equation*}
    A_i^{\mathrm{ScRew}} = \frac{S_i^{\bm{\lambda}}-\bar S^{\bm{\lambda}}}{\hat{\sigma}_{S^{\bm{\lambda}}}}, \quad \hat{\sigma}_{S^{\bm{\lambda}}} = \text{std}(S_1^{\bm{\lambda}}, \dots, S_G^{\bm{\lambda}}).
\end{equation*}
Subscript $i$ denotes the $i$-th sample in the group, and $\bar{\cdot}$ denotes the empirical mean over the group.
Let
\begin{equation*}
\psi_i := \sum_{t=0}^{T-1} \nabla_\theta \log \pi_\theta(a_t^i \mid s_t^i),
\end{equation*}
and define
\begin{equation*}
N_R := \frac{1}{G}\sum_{i=1}^G (R_i-\bar R)\psi_i,
\qquad
N_k := \frac{1}{G}\sum_{i=1}^G (C_{k,i}-\bar C_k)\psi_i.
\end{equation*}
Then the scalarized-reward constrained GRPO estimator can be written as
\begin{equation}
g^{\mathrm{ScRew}}(\bm{\lambda})
=
\frac{
\lambda_R N_R - \sum_{k=1}^K \lambda_k N_k
}{
\hat{\sigma}_{S^{\bm{\lambda}}}
}.
\end{equation}

We now compute the multiplier sensitivities of $g^{\mathrm{ScRew}}$:

\begin{theorem}[Multiplier sensitivity under scalarized rewards]
\label{thm:scalarized_reward_logit_sensitivity}
For a fixed sampled group, the sensitivity of the scalarized-reward GRPO estimator $g^{\mathrm{ScRew}}$ to the logit $z_k$ of constraint $k$ is
\begin{equation}
\frac{\partial g^{\mathrm{ScRew}}}{\partial z_k}
= \lambda_k \left(
\frac{\partial g^{\mathrm{ScRew}}}{\partial \lambda_k}
- \sum_{j \in \{R,1,\dots,K\}} \lambda_j \frac{\partial g^{\mathrm{ScRew}}}{\partial \lambda_j}
\right),
\end{equation}
where, letting $D := \lambda_R N_R - \sum_{k=1}^K \lambda_{k} N_{k}$, the partial sensitivities are
\begin{align*}
\frac{\partial g^{\mathrm{ScRew}}}{\partial \lambda_R}
= \phantom{-}\frac{N_R}{\hat\sigma_{S^{\bm{\lambda}}}}
- \frac{D}{\hat\sigma_{S^{\bm{\lambda}}}^3}\, \widehat{\mathrm{Cov}}(S^{\bm{\lambda}}, R), \quad\qquad \frac{\partial g^{\mathrm{ScRew}}}{\partial \lambda_k}
= -\frac{N_k}{\hat\sigma_{S^{\bm{\lambda}}}}
+ \frac{D}{\hat\sigma_{S^{\bm{\lambda}}}^3}\, \widehat{\mathrm{Cov}}(S^{\bm{\lambda}}, C_k),
\end{align*}
and $\widehat{\mathrm{Cov}}(S^{\bm{\lambda}}, X) := \frac{1}{G}\sum_{i=1}^G (S_i^{\bm{\lambda}} - \bar S^{\bm{\lambda}})(X_i - \bar X)$ denotes the empirical group covariance.
\end{theorem}

Proof in Appendix \ref{thm:scalarized_reward_logit_sensitivity_apdx}.
Observe that scalarized rewards induces \emph{two} forms of coupling. 
First, the softmax normalization means that increasing one logit decreases the effective weight on the other components by construction. 
Second, the shared denominator \(\hat\sigma_{S^{\bm{\lambda}}}\) introduces additional covariance-mediated coupling through \(\partial g^{\mathrm{ScRew}} /\partial \lambda_R \) and \(\partial g^{\mathrm{ScRew}} /\partial \lambda_k \). 
Consequently, changing one logit affects not only the corresponding constraint term, but also the relative contribution of the reward and all other constraints, 
while simultaneously modifying the normalization of the mixed signal; the resulting primal update is therefore altered both in magnitude and in direction.

\paragraph{Scalarized advantages.}
Our method instead normalizes each component separately and then combines the resulting standardized advantages:
\begin{equation*}
Z_{R,i} = \frac{R_i-\bar R}{\hat\sigma_R},
\qquad
Z_{C_k,i} = \frac{C_{k,i}-\bar C_k}{\hat\sigma_{C_k}}.
\end{equation*}
The scalarized advantage is
\begin{equation}
\label{eq:grpo_advantage_a1_multi}
A_i^{\mathrm{ScAdv}}
=
\lambda_R Z_{R,i}
-
\sum_{k=1}^K \lambda_k Z_{C_k,i}.
\end{equation}

The corresponding estimator is
\begin{equation}
g^{\mathrm{ScAdv}}(\bm{\lambda})
=
\lambda_R \frac{N_R}{\hat\sigma_R}
-
\sum_{k=1}^K
\lambda_k \frac{N_k}{\hat\sigma_{C_k}}.
\end{equation}

\begin{corollary}[Multiplier sensitivity under scalarized advantages]
\label{cor:scalarized_advantage_logit_sensitivity}
For a fixed sampled group, the sensitivity of the scalarized-advantage GRPO estimator to the logit \(z_k\) of constraint \(k\) is
\begin{equation}
\frac{\partial g^{\mathrm{ScAdv}}}{\partial z_k}
=
\lambda_k
\left(
-\frac{N_k}{\hat\sigma_{C_k}}
-
\lambda_R \frac{N_R}{\hat\sigma_R}
+
\sum_{k'=1}^K \lambda_{k'} \frac{N_{k'}}{\hat\sigma_{C_{k'}}}
\right).
\end{equation}
\end{corollary}

Unlike scalarized rewards, scalarized advantages does not introduce any additional shared-denominator coupling through the within-group normalization. With softmax-normalized multipliers, some competition between components is unavoidable, since increasing one logit reallocates mass away from the others. However, under scalarized advantages this is the \emph{only} source of coupling: the logit update is a competition over already-normalized component directions. In contrast, under scalarized rewards the same simplex competition is compounded by the dependence of \(\hat\sigma_S\) on the current reward-cost mixture, which adds covariance-mediated cross-coupling to the primal response.
For these reasons, we adopt scalarized advantages in constrained GRPO.

With scalarized advantages, the effective multipliers being applied to the advantage components is $\lambda_k$ by construction (i.e., Eq. \ref{eq:grpo_advantage_a1_multi}).
Whereas, Theorem~\ref{th:eff_mul_screw} shows that when scalarizing rewards, the multipliers are rescaled by $\tilde\lambda_k=\lambda_k\frac{\sigma_{C_k}}{\sigma_{S^{\bm{\lambda}}}}$ (likewise $\tilde\lambda_R$) such that
$A_i^{\mathrm{ScRew}} = \tilde\lambda_R Z_{R,i} -  \sum_{k=1}^K \tilde\lambda_k Z_{C_k,i}$.

\section{Experiments}

\subsection{Illustrative Environment}
\label{sec:toy_exp}

We begin with an illustrative example in a simple gridworld. 
An agent is randomly initialized in a $10\times10$ grid and must navigate to a randomly sampled goal location, receiving a sparse scalar reward upon reaching it. 
Each episode contains $20\%$ lava tiles; stepping onto lava incurs a penalty (i.e., a cost). 
In addition, the agent has a battery that depletes at every timestep and incurs a penalty whenever its battery level falls below $10\%$. 
The action space consists of five discrete actions: moving in the four cardinal directions or staying in place to recharge. 
This simple environment provides a controlled testbed for isolating the differences between \emph{scalarized rewards} and \emph{scalarized advantages} in constrained GRPO. Additional setup details and more experiments are provided in Appendix~\ref{sec:apdx:lava_env}.

We enforce constraint thresholds $\tilde d_{\text{lava}}=\tilde d_{\text{battery}}=0.01$ and train for $512$k episodes.
Figure~\ref{fig:gridworld_cgrpo} compares Constrained GRPO with scalarized advantages (ours) against a scalarized rewards baseline. We see that scalarized advantages yields more stable learning across seeds (5 seeds), as reflected by the narrower shaded region throughout training (column~1). 
Moreover, scalarized advantages leads to richer behavior-rate dynamics (column~2, bottom row): the Lagrangian multipliers (column~3, bottom row) adapt to allow the constraint rates to increase toward the threshold, effectively encouraging the policy to use its available cost budget. 
On the other hand, scalarized rewards rapidly drives both constraint rates close to zero early in training (column 2, top row), and keeps them near-zero thereafter, even as the learned multipliers (columb~3, top row) remain small.

This behavior can be understood through the implicit reweighting predicted by Theorem~\ref{th:eff_mul_screw}.
Observe that despite the small learned multipliers $\lambda_{\text{lava}}$ and $\lambda_{\text{battery}}$, the induced effective weights are larger and significantly noisier due to within-group variance and covariance effects (row~4, left column). This is especially problematic when the ordering of importance changes, as observed around 150K steps for the goal and lava weights (red/green swap). 
Overall, these results support \emph{scalarized advantages} as a more reliable advantage construction for Constrained GRPO, as it preserves the intended role of the Lagrangian multipliers in enforcing behavior specifications.

\begin{figure}[t]
    \centering
    \includegraphics[width=\columnwidth]{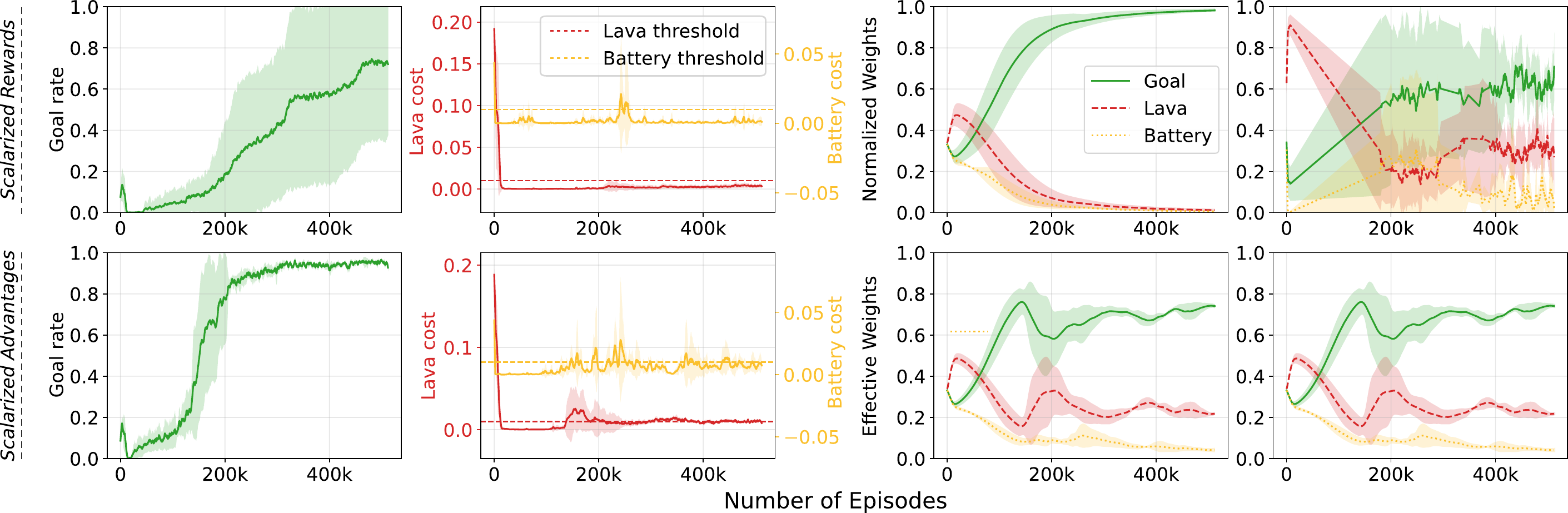}
    \captionsetup{font=footnotesize}
    \caption{Gridworld under Constrained GRPO, comparing \emph{scalarized rewards} (top) with \emph{scalarized advantages} (bottom). Scalarized advantages yields more stable training and uses the available cost budget, with behavior rates adapting toward the specified thresholds as the Lagrangian multipliers adjust. In contrast, scalarized rewards collapses constraint rates toward near-zero early in training with the effective weights remaining larger and noisier despite small learned multipliers.}
    \label{fig:gridworld_cgrpo}
    \vspace{-5mm}
\end{figure}

\subsection{NAVSIM Experiments}
\label{sec:DBS:navsim_exps}
\textbf{Benchmark Details.}

Our next experiments are on the NAVSIM-v2 driving benchmark~\cite{Cao2025CORL, Dauner2024NEURIPS}, which contains $103$k training scenarios (Navtrain) and $12$k testing scenarios (Navtest) curated from OpenScene~\cite{openscene2023,caesar2021nuplan}. 
NAVSIM-v2 evaluates planning using explicit behavioural safety constraints alongside task-level objectives such as progress and comfort, enabling targeted analysis of behaviour specification. We do our evaluations on the \emph{Navhard} split containing more challenging driving scenarios and planners are evaluated using a two-stage pseudo-simulation protocol:
Stage 1 evaluates the policy's ability to drive in unseen nominal driving scenarios while stage 2 evaluates the policy's ability to produce corrective behaviors.
More details are in Appendix~\ref{sec:apdx:navsim}.

The primary evaluation metric is Extended PDM Score (EPDMS $\in [0,1]$) which aggregates multiple subscores via a product of hard penalty terms and a weighted average of auxiliary terms (Eq. \ref{eq:apdx:epdms} in Appendix \ref{sec:apdx:navsim}).
The penalty terms, each taking values in $\{0,1\}$, include: (1) no at-fault collisions (NC), (2) drivable area compliance (DAC), (3) driving direction compliance (DDC), and (4) traffic light compliance (TLC). Since these penalties are combined multiplicatively, violating any one of them sets EPDMS$=0$ for that scene.
The weighted average terms consist of: (1) ego progress (EP) along the route, (2) time-to-collision (TTC), (3) history comfort (HC), (4) extended comfort (EC), and (5) lane keeping (LK).
We also report the $4$-second Average Displacement Error (ADE@4s) on Stage~1 scenes, computed with respect to the human trajectory.

Although the penalty terms are designed as hard constraints at evaluation time, directly optimizing a multiplicative hard-gated objective can lead to sparse learning signals and unstable optimization, as we show. 
We therefore treat penalty terms as soft behavioural constraints, targeting near-perfect satisfaction (0.99) while preserving informative gradients throughout training.

\textbf{Model Baselines.}
\hspace{.25cm}
We initialize all experiments with the Qwen2.5-VL 3B Instruct model~\cite{Bai2025Qwen25VLTR}, which we pre-train following the Poutine supervised finetuning (SFT) recipe~\cite{rowe2025poutine}. We then perform additional SFT on the NAVSIM training set~\cite{Cao2025CORL, Dauner2024NEURIPS} (see Appendix \ref{sec:apdx:navsim}), reserving $\approx11$k scenes as a held-out split for RL finetuning (RLFT). Our baselines are as follows:\footnote{Traffic light compliance (TLC) was almost always satisfied in preliminary experiments and is therefore omitted as an explicit objective, except when optimizing EPDMS.}

(1) \textit{GRPO with EPDMS}: optimize EPDMS as the reward, in the spirit of CaRL~\cite{jaeger2025carl} (which optimizes a similar objective with PPO). 
Because EPDMS is multiplicative, any NC/DAC/DDC/TLC violation collapses the reward to zero, acting as a hard constraint.

(2--5) \textit{Linear scalarization:} a weighted sum of constraints and the weighted-average terms (WAT): $R(\tau) = w_1\text{NC} + w_2\text{DAC} + w_3\text{DDC} + w_4\text{WAT}$.
This follows prior multi-objective GRPO setups~\cite{Li2025OptimizingSA,ichihara2025mo}.
We evaluate both scalarization choices: scalarize rewards (\textit{ScRew}) and scalarize advantages (\textit{ScAdv}) and consider two weighting schemes: (i) equal weights $w_i=0.25$ (\textit{GRPO, ScRew/ScAdv, Equal Weights}), 
and (ii) informed weights tuned from SFT behavior $w_1{=}0.1$, $w_2{=}0.8$, $w_3{=}0.05$, $w_4{=}0.05$ (\textit{GRPO, ScRew/ScAdv, Informed Weights}).

(6--7) \textit{Constrained GRPO:} treat WAT is as the main objective and enforce NC/DAC/DDC as constraints, setting all constraint satisfaction rate thresholds  to $0.99$~\footnote{Note that to maintain the upper-bound form of the constraint, we instead enforce a  $\tilde d_k=0.01$ threshold on the negation of the penalty terms NC/DAC/DDC.}. We again compare scalarized rewards and scalarized advantages (\textit{Constrained GRPO, ScRew} vs.\ \textit{Constrained GRPO, ScAdv}).

(8) We also compare against \textit{PPO and Constrained PPO}. The former uses EPDMS as the reward while the latter uses the standard Lagrangian variant of PPO \cite{ray2019benchmarking,stooke2020responsive} where the constraints are identical to Constrained GRPO described above.

\textbf{Results.}
Tables~\ref{tab:navsim_combined} reports results on NAVSIM-v2 Stage~1 and Stage~2, respectively.
For context, we also include the performance of the rule-based PDM planner~\cite{dauner2023parting}, whose trajectories are used to compute EP.
We now summarize the key findings:

\textbf{\textit{Behavior specification through constrained RL outperforms alternative strategies.}} Our results indicate that constrained RL provides a more effective mechanism for behavior specification than alternative reward-design approaches. In particular, the two constrained GRPO variants (bottom two rows of Table~\ref{tab:navsim_combined}) achieve strong overall performance, with constrained GRPO using scalarized advantages attaining the highest S1 EPDMS and second highest S2 EPDMS. On in-distribution Stage~1 driving scenarios, both constrained variants substantially improve upon the SFT checkpoint, whereas the unconstrained baselines do not consistently yield gains.
On the more challenging out-of-distribution Stage~2 split, all methods exhibit degraded performance relative to S1 scores, as expected. Nevertheless, our approach remains the most competitive among GRPO baselines and achieves comparable performance with PPO-Langagrian, showcasing the usefulness of our behavior specification framework for limiting bad driving behaviors.

\textit{\textbf{Scalarized advantages outperforms scalarizing rewards.}} 
Across all three linear formulations (equal weights, informed weights, and constrained), scalarizing advantages yields consistently higher EPDMS than scalarizing rewards, for both Stage~1 and Stage~2. This trend aligns with our analysis in Section~\ref{sec:constrained_grpo}: when objective components have mismatched scales, scalarizing rewards can implicitly bias the optimization toward high-variance terms, distorting the intended trade-offs. 
\begin{table*}[t]
\centering
\resizebox{\textwidth}{!}{
\begin{tabular}{l|c|c|cccccccc}
\toprule
Method & Constraints & Stage &
NC $\uparrow$ & DAC $\uparrow$ & DDC $\uparrow$ & TLC $\uparrow$ & EP $\uparrow$ & HC $\uparrow$ & EPDMS $\uparrow$ & ADE@4s $\downarrow$ \\
\midrule
\multirow{2}{*}{\textcolor{gray!70!black}{PDM-Closed (privileged) \cite{dauner2023parting}}} & \multirow{2}{*}{\textcolor{gray!70!black}{--}} & \textcolor{gray!70!black}{S1} & \textcolor{gray!70!black}{0.944} & \textcolor{gray!70!black}{0.988} & \textcolor{gray!70!black}{1.000} & \textcolor{gray!70!black}{0.995} & \textcolor{gray!70!black}{1.000} & \textcolor{gray!70!black}{0.877} & \textcolor{gray!70!black}{0.834} & \textcolor{gray!70!black}{--} \\
 & & \textcolor{gray!70!black}{S2} & \textcolor{gray!70!black}{0.881} & \textcolor{gray!70!black}{0.906} & \textcolor{gray!70!black}{0.963} & \textcolor{gray!70!black}{0.985} & \textcolor{gray!70!black}{1.000} & \textcolor{gray!70!black}{0.915} & \textcolor{gray!70!black}{0.687} & \textcolor{gray!70!black}{--} \\
\midrule\midrule

\multirow{2}{*}{SFT Checkpoint} & \multirow{2}{*}{--} & S1 & 0.953 & 0.789 & 0.978 & \underline{0.993} & 0.836 & \underline{0.976} & 0.651 & \textbf{2.539} \\
 & & S2 & 0.800 & 0.756 & 0.869 & 0.981 & 0.878 & 0.963 & 0.444 & -- \\
\midrule
\multirow{2}{*}{GRPO with EPDMS \cite{jaeger2025carl}} & \multirow{2}{*}{Hard} & S1 & 0.897 & 0.800 & 0.947 & 0.990 & 0.944 & \textbf{0.978} & 0.612 & 6.900 \\
 & & S2 & 0.749 & 0.759 & 0.829 & 0.969 & 0.979 & 0.945 & 0.381 & -- \\
\midrule
\multirow{2}{*}{GRPO, ScRew, Equal Weights \cite{Li2025OptimizingSA}} & \multirow{2}{*}{None} & S1 & 0.891 & 0.625 & 0.876 & 0.989 & \textbf{0.968} & \underline{0.976} & 0.445 & 9.050 \\
 & & S2 & 0.751 & 0.696 & 0.787 & 0.968 & \textbf{0.986} & 0.954 & 0.347 & -- \\
\midrule
\multirow{2}{*}{GRPO, ScAdv, Equal Weights \cite{ichihara2025mo}} & \multirow{2}{*}{None} & S1 & 0.908 & 0.661 & 0.900 & 0.989 & \underline{0.954} & 0.973 & 0.490 & 7.172 \\
 & & S2 & 0.767 & 0.705 & 0.813 & 0.971 & \underline{0.981} & 0.949 & 0.370 & -- \\
\midrule
\multirow{2}{*}{GRPO, ScRew, Informed Weights} & \multirow{2}{*}{None} & S1 & 0.887 & 0.643 & 0.887 & 0.987 & \textbf{0.967} & \textbf{0.978} & 0.463 & 8.509 \\
 & & S2 & 0.754 & 0.695 & 0.791 & 0.969 & \textbf{0.986} & 0.954 & 0.350 & -- \\
\midrule
\multirow{2}{*}{GRPO, ScAdv, Informed Weights} & \multirow{2}{*}{None} & S1 & 0.932 & 0.765 & 0.971 & \underline{0.993} & 0.890 & \underline{0.976} & 0.623 & 3.495 \\
 & & S2 & 0.778 & 0.731 & 0.831 & 0.972 & 0.944 & 0.962 & 0.385 & -- \\
\midrule
\multirow{2}{*}{PPO with EPDMS} & \multirow{2}{*}{Hard} & S1 & 0.944 & 0.731 & 0.945 & 0.991 & 0.862 & \underline{0.976} & 0.589 & 4.057 \\
 & & S2 & 0.794 & 0.727 & 0.855 & 0.969 & 0.908 & 0.963 & 0.415 & -- \\
\midrule
\multirow{2}{*}{PPO-Lagrangian \cite{ray2019benchmarking}} & \multirow{2}{*}{Soft} & S1 & \textbf{0.986} & \underline{0.891} & \textbf{0.994} & \textbf{0.998} & 0.721 & \textbf{0.978} & 0.737 & 4.558 \\
 & & S2 & \textbf{0.895} & \underline{0.812} & \textbf{0.912} & \textbf{0.990} & 0.767 & \textbf{0.980} & \textbf{0.536} & -- \\
\midrule\midrule
\multirow{2}{*}{Constrained GRPO, ScRew} & \multirow{2}{*}{Soft} & S1 & \underline{0.955} & 0.887 & \underline{0.987} & \underline{0.993} & 0.868 & \underline{0.976}  & \underline{0.743} & \underline{3.344} \\
 & & S2 & 0.810 & 0.784 & 0.866 & 0.976 & 0.939 & 0.952 & 0.436 & -- \\
\midrule
\multirow{2}{*}{\textbf{Constrained GRPO, ScAdv (ours)}} & \multirow{2}{*}{Soft} & S1 & \underline{0.957} & \textbf{0.928} & \underline{0.989} & \textbf{0.998} & 0.796 & \textbf{0.977} & \textbf{0.774} & 3.361 \\
 & & S2 & \underline{0.853} & \textbf{0.830} & \underline{0.897} & \underline{0.986} & 0.863 & \underline{0.967} & \underline{0.516} & -- \\
\bottomrule
\end{tabular}
}
\captionsetup{font=footnotesize}
\caption{\textbf{Results on NAVSIM-v2 Navhard.} We compare several baselines with Constrained GRPO in Stage 1 (S1, real scenes) and Stage 2 (S2, synthetic scenes). The Constraints column indicates whether NC/DAC/DDC/TLC are treated as hard constraints, soft constraints, or not enforced. \textit{ScRew}: Scalarized Rewards; \textit{ScAdv}: Scalarized Advantages.}
\label{tab:navsim_combined}
\vspace{-5mm}
\end{table*}

We further illustrate this effect in Figure~\ref{fig:train_constrained_multipliers_grpo} in Appendix~\ref{sec:apdx:navsim}. 
Under constrained GRPO with ScRew, the learned Lagrangian multipliers do not translate into effective constraint enforcement. This is most evident for DAC and EP, which is a proxy for the WAT objective: despite DAC being assigned nearly the maximum multiplier weight (and the main reward weight approaching zero), DAC does not reach the target satisfaction level, while EP continues to increase and only slowly decays later in training.
In contrast, constrained GRPO with ScAdv responds rapidly to multiplier updates. 
DAC improves within the first $\sim$1000 training steps, while EP decreases as its multiplier is reduced, demonstrating that the Lagrangian weights exert their intended influence on the learning dynamics. 
Finally, around $\sim$2500 steps, the NC multiplier increases sharply as DAC becomes satisfactory but collision compliance degrades, showing that the optimizer adaptively reallocates emphasis across constraints. 

\textit{\textbf{Hard constraints lead to suboptimal learning dynamics.}} Finally, observe that directly optimizing EPDMS as a reward (i.e., GRPO with EPDMS) underperforms the soft-constraint formulations, yielding lower overall EPDMS on both Stage~1 and Stage~2. Intuitively, EPDMS couples constraints through a multiplicative structure, collapsing the learning signal whenever any penalty term is violated. This removes the ability to trade off and prioritize objectives smoothly during training, and prevents variance-aware balancing of component signals. In contrast, our soft-constraint formulation retains informative gradients even when constraints are partially violated, enabling more stable optimization and ultimately improved behavior specification.

\subsection{GSM8K Experiments}
\label{sec:DBS:math_exps}

\textbf{Benchmark Details.} \hspace{.25cm}
We further evaluate our approach in a language reasoning setting using the GSM8K dataset, which consists of grade-school math word problems requiring multi-step reasoning.
For each completion, we consider four signals: correctness, format adherence, integer validity, and length (which favors shorter reasoning chains).

\textbf{Model Baselines.}
\hspace{.25cm}
We initialize all experiments from Qwen2.5-Instruct models at three scales: $1.5$B, $3$B, and $7$B parameters.
For the baselines, we compare to \textit{GDPO} and \textit{PPO-Lagrangian}.
GDPO~\cite{liu2026gdpo} is a prior advantage-scalarization baseline designed for multi-objective GRPO training.
PPO-Lagrangian uses a learned value function and enforces constraints through the standard primal-dual PPO formulation.
For our constrained formulation, we treat correctness, format adherence, and integer validity as constraints, and optimize the length objective as the primary reward. 
Our goal is not to claim that correctness is only required up to a threshold; rather, the thresholds serve as optimization targets that prevent the policy from trivially improving length by sacrificing correctness or output validity.

\begin{table*}[t]
\centering
\resizebox{\textwidth}{!}{
\begin{tabular}{l|c|cccc}
\toprule
Method & \# params & Length $\downarrow$ & Format $\uparrow$ & Integer $\uparrow$ & Correctness $\uparrow$ \\
\cmidrule(r){1-2}\cmidrule(l){3-6}
PPO-Lagrangian & 1.5B & 346.93 & 0.886 & 0.935 & 0.604 \\
GDPO~\cite{liu2026gdpo} & 1.5B & 296.89 & \textbf{0.953} & \textbf{0.981} & 0.629 \\
Constrained GRPO, ScRew & 1.5B & \textbf{246.82} & 0.934 & 0.941 & 0.640 \\
Constrained GRPO, ScAdv (ours) & 1.5B & 293.36 & \textbf{0.953} & 0.966 & \textbf{0.712} \\
\midrule
PPO-Lagrangian & 3B & 753.10 & 0.759 & 0.867 & 0.642 \\
GDPO~\cite{liu2026gdpo} & 3B & \textbf{330.11} & 0.911 & 0.977 & 0.770 \\
Constrained GRPO, ScRew & 3B & \underline{334.61} & 0.965 & 0.968 & 0.751 \\
Constrained GRPO, ScAdv (ours) & 3B & \underline{349.93} & \textbf{0.987} & \textbf{0.994} & \textbf{0.802} \\
\midrule
PPO-Lagrangian & 7B & 327.74 & 0.980 & \underline{0.995} & \textbf{0.891} \\
GDPO~\cite{liu2026gdpo} & 7B & \underline{127.49} & \textbf{0.990} & \underline{0.995} & 0.768 \\
Constrained GRPO, ScRew & 7B & \textbf{119.65} & \underline{0.988} & \underline{0.993} & 0.731 \\
Constrained GRPO, ScAdv (ours) & 7B & 256.62 & \underline{0.989} & \textbf{0.996} & \underline{0.887} \\
\bottomrule
\end{tabular}
}
\captionsetup{font=footnotesize}
\caption{\textbf{GSM8K results across model scales.} We compare \textit{PPO-Lagrangian}, \textit{GDPO~\cite{liu2026gdpo}}, and constrained GRPO with \textit{scalarized rewards} and \textit{scalarized advantages} with Qwen2.5-Instruct models at $1.5$B, $3$B, and $7$B parameters. 
Constrained GRPO consistently improves correctness over GDPO at all scales while maintaining strong format and integer validity, and achieves competitive or better performance at $1.5$B and $3$B without requiring a learned critic (i.e., PPO-Lagrangian). Bold indicates the best;  Underlined values indicate results that are close to the best. 
}
\label{tab:gsm8k_main}
\vspace{-5mm}
\end{table*}

\textbf{Results.}
\hspace{.25cm}
Tables~\ref{tab:gsm8k_main} reports our GSM8K results.
Across all three model scales, constrained GRPO achieves higher correctness than GDPO while maintaining strong performance on format adherence and integer validity.
The gains are most pronounced at larger scale: at $3$B, constrained GRPO improves correctness from $0.770$ to $0.802$, and at $7$B from $0.768$ to $0.887$.
These results suggest that explicitly enforcing correctness as a constraint prevents the model from over-optimizing easier auxiliary objectives such as reducing output length.
We observe that GDPO often produces substantially shorter generations, especially at the $7$B scale, where it reduces length to $127.49$ tokens on average but at the expense of correctness. 
Compared to PPO-Lagrangian, constrained GRPO achieves comparable or better performance on all four metrics at the $1.5$B and $3$B scales.
At $7$B, PPO-Lagrangian attains slightly higher correctness ($0.891$ vs.\ $0.887$), but constrained GRPO remains competitive and produces shorter outputs without requiring expensive critic training.

\section{Conclusion}
\label{sec:conclusion}

We introduced Constrained GRPO, a Lagrangian extension of GRPO that enforces user-specified constraints via indicator costs and learned multipliers. 
We showed that scalarizing rewards before normalization introduces shared-denominator coupling, so that changing one multiplier can affect not only its corresponding constraint term, but also the relative contribution of the other terms to the policy update. 
By instead scalarizing standardized advantages, Constrained GRPO yields a better-conditioned update and more stable multiplier dynamics in practice.

\section{Acknowledgements}
We thank Samsung, the IVADO and the Canada First
Research Excellence Fund (CFREF) / Apogee Funds,
the Canada CIFAR AI Chairs Program, Fonds de Recherche Nature et Technologies (FRQNT), and the NSERC
Discovery Grants program for their financial support. We thank Mila - Quebec AI Institute for compute resources. 
This research was supported by grants from NVIDIA and utilized NVIDIA 15K A100 GPU-hours on Saturn Cloud.

\bibliographystyle{unsrtnat}  %
\bibliography{neurips}     %

\appendix
\onecolumn

\section{Additional theorems and Proofs}
\label{sec:apdx:proof}

\begin{theorem}[Theorem~\ref{thm:scalarized_reward_logit_sensitivity} restated]
\label{thm:scalarized_reward_logit_sensitivity_apdx}
For a fixed sampled group, the sensitivity of the scalarized-reward GRPO estimator $g^{\mathrm{ScRew}}$ to the logit $z_k$ of constraint $k$ is
\begin{equation}
\frac{\partial g^{\mathrm{ScRew}}}{\partial z_k}
= \lambda_k \left(
\frac{\partial g^{\mathrm{ScRew}}}{\partial \lambda_k}
- \sum_{j \in \{R,1,\dots,K\}} \lambda_j \frac{\partial g^{\mathrm{ScRew}}}{\partial \lambda_j}
\right),
\end{equation}
where, letting $D := \lambda_R N_R - \sum_{k=1}^K \lambda_{k} N_{k}$, the partial sensitivities are
\begin{align*}
\frac{\partial g^{\mathrm{ScRew}}}{\partial \lambda_R}
= \frac{N_R}{\hat\sigma_{S^{\bm{\lambda}}}}
- \frac{D}{\hat\sigma_{S^{\bm{\lambda}}}^3}\, \widehat{\mathrm{Cov}}(S^{\bm{\lambda}}, R), \quad\qquad \frac{\partial g^{\mathrm{ScRew}}}{\partial \lambda_k}
= -\frac{N_k}{\hat\sigma_{S^{\bm{\lambda}}}}
+ \frac{D}{\hat\sigma_{S^{\bm{\lambda}}}^3}\, \widehat{\mathrm{Cov}}(S^{\bm{\lambda}}, C_k),
\end{align*}
and $\widehat{\mathrm{Cov}}(S^{\bm{\lambda}}, X) := \frac{1}{G}\sum_{i=1}^G (S_i^{\bm{\lambda}} - \bar S^{\bm{\lambda}})(X_i - \bar X)$ denotes the empirical group covariance.
\end{theorem}

\begin{proof}
\label{th:proof_screw_logit_sens}
For a fixed sampled group, define (as before)
\[
g^{\mathrm{ScRew}}(\bm{\lambda})
=
\frac{D}{\hat\sigma_{S^{\bm{\lambda}}}},
\qquad
N_R := \frac{1}{G}\sum_{i=1}^G (R_i-\bar R)\psi_i,
\qquad
N_k := \frac{1}{G}\sum_{i=1}^G (C_{k,i}-\bar C_k)\psi_i.
\]
Since the sampled group is fixed, the quantities \(N_R\) and \(\{N_k\}_{k=1}^K\) do not depend on \(\bm{\lambda}\). The dependence on \(\bm{\lambda}\) enters through the linear numerator \(D\) and through the denominator \(\hat\sigma_{S^{\bm{\lambda}}}\).

We first compute the partial sensitivities with respect to the multiplier coordinates. Recall that
\[
S_i^{\bm{\lambda}} = \lambda_R R_i - \sum_{k=1}^K \lambda_{k} C_{k,i},
\qquad
\bar S^{\bm{\lambda}} = \frac{1}{G}\sum_{i=1}^G S_i^{\bm{\lambda}}.
\]

\paragraph{Derivative with respect to \(\lambda_k\).}
For any constraint \(k\),
\[
\frac{\partial S_i^{\bm{\lambda}}}{\partial \lambda_k} = -C_{k,i},
\qquad
\frac{\partial \bar S^{\bm{\lambda}}}{\partial \lambda_k} = -\bar C_k,
\]
so
\[
\frac{\partial (S_i^{\bm{\lambda}}-\bar S^{\bm{\lambda}})}{\partial \lambda_k}
=
-(C_{k,i}-\bar C_k).
\]
Now write
\[
\hat\sigma_{S^{\bm{\lambda}}}
=
\sqrt{
\frac{1}{G}\sum_{i=1}^G (S_i^{\bm{\lambda}}-\bar S^{\bm{\lambda}})^2} \ .
\]
Differentiating gives,
\begin{align}
\frac{\partial \hat\sigma_{S^{\bm{\lambda}}}}{\partial \lambda_k}
&=
\frac{1}{2\hat\sigma_{S^{\bm{\lambda}}}}
\frac{\partial}{\partial \lambda_k}
\left(
\frac{1}{G}\sum_{i=1}^G (S_i^{\bm{\lambda}}-\bar S^{\bm{\lambda}})^2
\right) \nonumber\\
&=
-\frac{1}{\hat\sigma_{S^{\bm{\lambda}}}}
\frac{1}{G}\sum_{i=1}^G
(S_i^{\bm{\lambda}}-\bar S^{\bm{\lambda}})(C_{k,i}-\bar C_k) \nonumber\\
&=
-\frac{\widehat{\mathrm{Cov}}(S^{\bm{\lambda}},C_k)}{\hat\sigma_{S^{\bm{\lambda}}}}.
\label{eq:proof_sigma_constraint}
\end{align}
Also,
\begin{align}
\frac{\partial g^{\mathrm{ScRew}}}{\partial \lambda_k}
&=
\frac{1}{\hat\sigma_{S^{\bm{\lambda}}}}\frac{\partial D}{\partial \lambda_k}
+
D \frac{\partial}{\partial \lambda_k}
\left(
\frac{1}{\hat\sigma_{S^{\bm{\lambda}}}}
\right) \nonumber\\
&=
-\frac{N_k}{\hat\sigma_{S^{\bm{\lambda}}}}
-
\frac{D}{\hat\sigma_{S^{\bm{\lambda}}}^2}
\frac{\partial \hat\sigma_{S^{\bm{\lambda}}}}{\partial \lambda_k}.
\label{eq:proof_constraint_quotient}
\end{align}
Substituting \eqref{eq:proof_sigma_constraint} into \eqref{eq:proof_constraint_quotient} gives
\begin{equation}
\label{eq:nosoftnorm_proof_screw}
\frac{\partial g^{\mathrm{ScRew}}}{\partial \lambda_k}
=
-\frac{N_k}{\hat\sigma_{S^{\bm{\lambda}}}}
+
\frac{D}{\hat\sigma_{S^{\bm{\lambda}}}^3}\,
\widehat{\mathrm{Cov}}(S^{\bm{\lambda}},C_k).
\end{equation}

Note that \eqref{eq:nosoftnorm_proof_screw} shows the sensitivity of the policy gradient with respect to multipliers $\lambda_i$ if these were not softmax-normalized.

The derivative with respect to \(\lambda_R\) follows a near identical derivation:
\paragraph{Derivative with respect to \(\lambda_R\).}
We have
\[
\frac{\partial (S_i^{\bm{\lambda}}-\bar S^{\bm{\lambda}})}{\partial \lambda_R}
=
R_i-\bar R.
\]
Differentiating gives
\begin{align*}
\frac{\partial \hat\sigma_{S^{\bm{\lambda}}}}{\partial \lambda_R}
&=
\frac{1}{2\hat\sigma_{S^{\bm{\lambda}}}}
\frac{\partial}{\partial \lambda_R}
\left(
\frac{1}{G}\sum_{i=1}^G (S_i^{\bm{\lambda}}-\bar S^{\bm{\lambda}})^2
\right) \nonumber\\
&=
\frac{1}{\hat\sigma_{S^{\bm{\lambda}}}}
\frac{1}{G}\sum_{i=1}^G
(S_i^{\bm{\lambda}}-\bar S^{\bm{\lambda}})(R_i-\bar R) \nonumber\\
&=
\frac{\widehat{\mathrm{Cov}}(S^{\bm{\lambda}},R)}{\hat\sigma_{S^{\bm{\lambda}}}}.
\end{align*}

and so
\begin{align}
\label{eq:nosoftnorm_proof_screw_reward}
\frac{\partial g^{\mathrm{ScRew}}}{\partial \lambda_R}
&=
\frac{1}{\hat\sigma_{S^{\bm{\lambda}}}}\frac{\partial D}{\partial \lambda_R}
+
D \frac{\partial}{\partial \lambda_R}
\left(
\frac{1}{\hat\sigma_{S^{\bm{\lambda}}}}
\right) \nonumber\\
&=
\frac{N_R}{\hat\sigma_{S^{\bm{\lambda}}}}
-
\frac{D}{\hat\sigma_{S^{\bm{\lambda}}}^2}
\frac{\partial \hat\sigma_{S^{\bm{\lambda}}}}{\partial \lambda_R} \nonumber\\
&=
\frac{N_R}{\hat\sigma_{S^{\bm{\lambda}}}}
-
\frac{D}{\hat\sigma_{S^{\bm{\lambda}}}^3}\,
\widehat{\mathrm{Cov}}(S^{\bm{\lambda}},R).
\end{align}

\paragraph{Chain rule through the multiplier logits.}
Since the multipliers are softmax-normalized,
\[
\lambda_j = \frac{\exp(z_j)}{\sum_{j' \in \{R,1,\dots,K\}} \exp(z_{j'})},
\]
their Jacobian is
\[
\frac{\partial \lambda_j}{\partial z_k}
=
\lambda_j(\delta_{jk}-\lambda_k),
\qquad
j \in \{R,1,\dots,K\}.
\]
Applying chain rule,
\begin{align}
\label{eq:grad_softmax_proof}
\frac{\partial g^{\mathrm{ScRew}}}{\partial z_k}
&=
\sum_{j \in \{R,1,\dots,K\}}
\frac{\partial g^{\mathrm{ScRew}}}{\partial \lambda_j}
\frac{\partial \lambda_j}{\partial z_k} \nonumber\\
&=
\sum_{j \in \{R,1,\dots,K\}}
\frac{\partial g^{\mathrm{ScRew}}}{\partial \lambda_j}
\,\lambda_j(\delta_{jk}-\lambda_k) \nonumber\\
&=
\lambda_k \frac{\partial g^{\mathrm{ScRew}}}{\partial \lambda_k}
-
\lambda_k \sum_{j \in \{R,1,\dots,K\}}
\lambda_j \frac{\partial g^{\mathrm{ScRew}}}{\partial \lambda_j}.
\end{align}
Factoring out \(\lambda_k\) yields
\begin{equation}
\label{eq:proof_partial_z_k}
\frac{\partial g^{\mathrm{ScRew}}}{\partial z_k}
= \lambda_k \left(
\frac{\partial g^{\mathrm{ScRew}}}{\partial \lambda_k}
- \sum_{j \in \{R,1,\dots,K\}} \lambda_j \frac{\partial g^{\mathrm{ScRew}}}{\partial \lambda_j}
\right),
\end{equation}
and combining \eqref{eq:nosoftnorm_proof_screw}, \eqref{eq:nosoftnorm_proof_screw_reward} with \eqref{eq:proof_partial_z_k} proves the result.
\end{proof}

\begin{corollary}[Corollary~\ref{cor:scalarized_advantage_logit_sensitivity} restated]
\label{cor:scalarized_advantage_logit_sensitivity_apdx}
For a fixed sampled group, the sensitivity of the scalarized-advantage GRPO estimator to the logit \(z_k\) of constraint \(k\) is
\begin{equation}
\frac{\partial g^{\mathrm{ScAdv}}}{\partial z_k}
=
\lambda_k
\left(
-\frac{N_k}{\hat\sigma_{C_k}}
-
\lambda_R \frac{N_R}{\hat\sigma_R}
+
\sum_{k'=1}^K \lambda_{k'} \frac{N_{k'}}{\hat\sigma_{C_{k'}}}
\right).
\end{equation}
\end{corollary}

\begin{proof}
\label{cor:proof_scadv_logit_sense}
For a fixed sampled group, the scalarized-advantage GRPO estimator can be written as
\[
g^{\mathrm{ScAdv}}(\bm{\lambda})
=
\lambda_R \frac{N_R}{\hat\sigma_R}
-
\sum_{k=1}^K \lambda_{k} \frac{N_{k}}{\hat\sigma_{C_{k}}}.
\]
Since the sampled group is fixed, the quantities \(N_R\), \(N_{k}\), \(\hat\sigma_R\), and \(\hat\sigma_{C_{k}}\) do not depend on \(\bm{\lambda}\). Therefore, the partial sensitivities with respect to the multiplier coordinates are immediate:
\begin{equation}
\label{eq:nosoftnorm_proof}
\frac{\partial g^{\mathrm{ScAdv}}}{\partial \lambda_R}
=
\frac{N_R}{\hat\sigma_R},
\qquad
\frac{\partial g^{\mathrm{ScAdv}}}{\partial \lambda_k}
=
-\frac{N_k}{\hat\sigma_{C_k}}.
\end{equation}

Note that \eqref{eq:nosoftnorm_proof} shows the sensitivity of the policy gradient with respect to multipliers $\lambda_i$ if these were not softmax-normalized. However, since
\[
\lambda_j = \frac{\exp(z_j)}{\sum_{j' \in \{R,1,\dots,K\}} \exp(z_{j'})},
\]
Then combining \eqref{eq:grad_softmax_proof} with \eqref{eq:nosoftnorm_proof}, we get
\[
\frac{\partial g^{\mathrm{ScAdv}}}{\partial z_k}
=
\lambda_k
\left(
-\frac{N_k}{\hat\sigma_{C_k}}
-
\lambda_R \frac{N_R}{\hat\sigma_R}
+
\sum_{k'=1}^K \lambda_{k'} \frac{N_{k'}}{\hat\sigma_{C_{k'}}}
\right),
\]
which proves the result.
\end{proof}

\newpage
\begin{theorem}[Effective Multiplier for Scalarized Rewards]
\label{th:eff_mul_screw}

Let
\begin{equation*}
\label{eq:apdx:sr_components_thm}
\mathbf{x} \;:=\; 
\begin{bmatrix}
R \\ C_1 \\ \vdots \\ C_K
\end{bmatrix}
\in \mathbb{R}^{K+1},
\qquad
\boldsymbol{\lambda} \;:=\;
\begin{bmatrix}
\lambda_R \\ -\lambda_1 \\ \vdots \\ -\lambda_K
\end{bmatrix}
\in \mathbb{R}^{K+1},
\end{equation*}
and define the scalarized reward $S^{\bm{\lambda}} := \boldsymbol{\lambda}^\top \mathbf{x}$. 
Denote the within-group mean and covariance of $\mathbf{x}$ by $\boldsymbol{\mu}:=\mathbb{E}[\mathbf{x}]$ and $\mathbf{\Sigma}:=\mathrm{Cov}(\mathbf{x})$, and let
$\sigma_{S^{\bm{\lambda}}} \;:=\; \sqrt{\mathrm{Var}(S^{\bm{\lambda}})} \;=\; \sqrt{\boldsymbol{\lambda}^\top \mathbf{\Sigma}\,\boldsymbol{\lambda}}$.
Further define per-component standardizations
\begin{equation*}
\label{eq:apdx:sr_zscores_thm}
\begin{aligned}
\mathbf{z} &:= \mathbf{D}^{-1}(\mathbf{x}-\boldsymbol{\mu}),
\qquad
\mathbf{D} := \mathrm{diag}(\sigma_R,\sigma_{C_1},\dots,\sigma_{C_K}) 
\end{aligned}
\end{equation*}
so that 
\begin{equation}
\label{eq:apdx:per_comp_stand}
\begin{aligned}
Z_R &:= \frac{R-\mu_R}{\sigma_R},\\ \quad
Z_{C_k} &:= \frac{C_k-\mu_{C_k}}{\sigma_{C_k}}, \quad k = 1,\dots,K,
\end{aligned}
\end{equation}
where $\mu_{(\cdot)}$ and $\sigma_{(\cdot)}$ denote the within-group mean and standard deviation of the corresponding component.
Then the GRPO advantage obtained by standardizing the scalarized reward,
\begin{equation*}
\label{eq:apdx:sr_advantage_thm}
A^{\mathrm{ScRew}} \;:=\; \frac{S^{\bm{\lambda}}-\mathbb{E}[S^{\bm{\lambda}}]}{\sigma_{S^{\bm{\lambda}}}},
\end{equation*}
admits the expansion
\begin{equation}
\label{eq:apdx:sr_advantage_expanded_thm}
A^{\mathrm{ScRew}}
\;=\;
\frac{(\mathbf{D}\boldsymbol{\lambda})^\top \mathbf{z}}{\sigma_{S^{\bm{\lambda}}}}
\;=\; 
\sum_{j=0}^{K}\Big(\frac{\lambda_j\,\sigma_j}{\sigma_{S^{\bm{\lambda}}}}\Big)\,Z_j,
\end{equation}
where we index $j=0$ as the main reward component ($\lambda_0\!\equiv\!\lambda_R$, $\sigma_0\!\equiv\!\sigma_R$) and $j=k$ corresponds to cost $C_k$.
Consequently, the effective contribution of each standardized component $Z_j$ is scaled by $\lambda_j\sigma_j/\sigma_{S^{\bm{\lambda}}}$, and therefore depends on within-group scales and (through $\sigma_{S^{\bm{\lambda}}}$) the full covariance structure $\mathbf{\Sigma}$.
\end{theorem}

\begin{proof}
By definition, $S^{\bm{\lambda}}=\boldsymbol{\lambda}^\top \mathbf{x}$ and $\mathbb{E}[S^{\bm{\lambda}}]=\boldsymbol{\lambda}^\top \boldsymbol{\mu}$, hence
\begin{equation}
\label{eq:apdx:sr_centering_proof}
S^{\bm{\lambda}}-\mathbb{E}[S^{\bm{\lambda}}] \;=\; \boldsymbol{\lambda}^\top(\mathbf{x}-\boldsymbol{\mu}).
\end{equation}
Moreover,
\begin{equation*}
\label{eq:apdx:sr_var_proof}
\mathrm{Var}(S^{\bm{\lambda}})
= \mathrm{Var}(\boldsymbol{\lambda}^\top \mathbf{x})
= \boldsymbol{\lambda}^\top \mathrm{Cov}(\mathbf{x})\,\boldsymbol{\lambda}
= \boldsymbol{\lambda}^\top \mathbf{\Sigma}\,\boldsymbol{\lambda},
\end{equation*}
which yields $\sigma_{S^{\bm{\lambda}}}=\sqrt{\boldsymbol{\lambda}^\top \mathbf{\Sigma}\,\boldsymbol{\lambda}}$.

Substituting into \eqref{eq:apdx:sr_centering_proof} gives
\begin{equation*}
\label{eq:apdx:sr_substitution_proof}
S^{\bm{\lambda}}-\mathbb{E}[S^{\bm{\lambda}}] \;=\; \boldsymbol{\lambda}^\top \mathbf{D}\mathbf{z} \;=\; (\mathbf{D}\boldsymbol{\lambda})^\top \mathbf{z}.
\end{equation*}
Dividing both sides by $\sigma_{S^{\bm{\lambda}}}$ yields
\begin{equation*}
A^{\mathrm{ScRew}}
= \frac{S^{\bm{\lambda}}-\mathbb{E}[S^{\bm{\lambda}}]}{\sigma_{S^{\bm{\lambda}}}}
= \frac{(\mathbf{D}\boldsymbol{\lambda})^\top \mathbf{z}}{\sigma_{S^{\bm{\lambda}}}},
\end{equation*}
which is the first equality in \eqref{eq:apdx:sr_advantage_expanded_thm}. Expanding the dot product gives the summation form
\begin{equation*}
\label{eq:apdx:sr_advantage_expanded}
A^{\mathrm{ScRew}}
= \sum_{j=0}^{K}\Big(\frac{\lambda_j\,\sigma_j}{\sigma_{S^{\bm{\lambda}}}}\Big)\,Z_j,
\end{equation*}
completing the proof.
\end{proof}

\newpage
\section{Preliminaries}
\label{sec:mdp_prelims}

\textbf{Markov Decision Processes (MDPs).}
A Markov decision process (MDP) is defined by $(\mathcal{S},\mathcal{A},P,R)$~\cite{sutton2018reinforcement}, where $\mathcal{S}$ is the state space, $\mathcal{A}$ the action space, $P(\cdot\mid s,a)$ the transition kernel, and $R(s,a)$ the reward function. 
At time $t$, the agent samples an action $a_t\sim\pi(\cdot\mid s_t)$, transitions according to $s_{t+1}\sim P(\cdot\mid s_t,a_t)$, and receives reward $R(s_t,a_t)$. 
The trajectory distribution under $\pi$ and the expected discounted return are
\begin{align*}
p_\pi(\tau) &= P_0(s_0)\prod_{t=0}^{T-1} P(s_{t+1}\mid s_t,a_t)\,\pi(a_t\mid s_t), \\
J_R(\pi) &= \mathbb{E}_{\tau\sim p_\pi}\Big[\sum_{t=0}^{T-1} \gamma^t R(s_t,a_t)\Big],
\end{align*}
where the objective is to learn a policy $\pi^*$ that maximizes the expected discounted return, i.e. $\pi^*=\arg\max_{\pi} J_R(\pi)$.

\textbf{Constrained Markov Decision Processes (CMDPs).}
A constrained MDP (CMDP) augments the MDP objective with $K$ expected cost constraints~\cite{altman1999constrained}. 
Given cost functions $C_k:\mathcal{S}\times\mathcal{A}\to\mathbb{R}$ and thresholds $d_k$, define the feasible set of policies as
\begin{equation*}
\begin{aligned}
\Pi_C &= \Big\{\pi\in\Pi:\ J_{C_k}(\pi)\le d_k,\;\;k=1,\dots,K\Big\}, \\
J_{C_k}(\pi) &= \mathbb{E}_{\tau\sim p_\pi}\Big[\sum_{t=0}^{T-1} \gamma^t C_k(s_t,a_t)\Big].
\end{aligned}
\end{equation*}

The constrained optimal policy maximizes return subject to feasibility:
\begin{align*}
\pi^*=\arg\max_{\pi\in\Pi}\ J_R(\pi)\quad \text{s.t.}\quad J_{C_k}(\pi)\le d_k,\;\;k=1,\dots,K,
\end{align*}
where stationary \emph{stochastic} policies may be required for optimality~\cite{altman1999constrained}.

\section{Gridworld Environment}
\label{sec:apdx:lava_env}

\subsection{Environment Details}

Our gridworld is a $10\times10$ grid in which the agent is initialized uniformly at random and must navigate to a randomly sampled goal location. 
Reaching the goal yields a sparse terminal reward of $1.0$. 
At the start of each episode, $20\%$ of the tiles are designated as lava; stepping onto a lava tile incurs a penalty of $1.0$ (i.e., a cost). 
The agent is also equipped with a battery that drains by $5\%$ at every timestep and incurs a penalty whenever its charge falls below $10\%$. 
The action space consists of five discrete actions: moving in the four cardinal directions, or staying in place to recharge at a rate of $20\%$ per timestep.
The agent observes a compact state representation consisting of:
\begin{itemize}
    \item its $(x,y)$ position, normalized to $[0,1]^2$;
    \item the goal $(x,y)$ coordinates, normalized to $[0,1]^2$;
    \item the current battery level in $[0,1]$;
    \item a $5\times5$ local occupancy map centered on the agent, where entries indicate lava tiles (flattened to a 25-dimensional vector).
\end{itemize}
Episodes last at most $80$ timesteps and terminate early only upon reaching the goal.

\subsection{Implementation Details}

We train a categorical policy parameterized by a two-layer MLP (hidden size $128$ with ReLU activations) that outputs action logits. 
Policies are optimized with the GRPO clipped objective ($\epsilon=0.2$) using $2$ epochs per update, minibatches of size $2048$, and an entropy bonus of $0.001$. 
Training runs for $8000$ update steps; each update collects $64$ episodes via grouped rollouts (group size $8$ across $8$ groups). 
We use Adam for the policy optimizer (learning rate $5\times 10^{-4}$ by default) and, when constraints are enabled, a separate Adam optimizer for cost weights (learning rate $10^{-2}$). 
Cost weights are parameterized via a softmax over \{reward, lava, battery\} parameters and updated every step and initialized to $0.02$ logits. 
All experiments were run on 16 CPU cores, and each run type required approximately 3 hours to complete across all five seeds.

\subsection{Additional Results}
We compare \emph{scalarized rewards} and \emph{scalarized advantages} under \emph{fixed} reward and penalty weights to illustrate the benefit of scalarizing advantages in the unconstrained setting \cite{ichihara2025mo}. 
For simplicity, we restrict the battery-penalty weight to $\lambda_{\text{battery}}^{\text{fixed}} \in \{0.0, 0.1\}$ and sweep the lava-penalty weight over $\lambda_{\text{lava}}^{\text{fixed}} \in \{0.0, 0.01, 0.05, 0.1, 0.2, 0.3, 0.5, 0.7, 1.0\}$. 
Figure~\ref{fig:gridworld_fixed_weights} illustrates how performance varies as we sweep $\lambda_{\text{lava}}^{\text{fixed}}$ for the two fixed settings of $\lambda_{\text{battery}}^{\text{fixed}}$.
Overall, we observe consistently higher goal-reaching rates with \emph{scalarized advantages}, along with more stable performance across seeds (through smaller error bars). 
Moreover, \emph{scalarized rewards} tends to unilaterally optimize the lava rate: even the smallest non-zero lava weight, $\lambda_{\text{lava}}^{\text{fixed}}=0.01$, leads to near-minimal lava visitation. 
In contrast, \emph{scalarized advantages} provides smoother, more gradual control over the lava rate, which we argue better reflects the intended meaning of the weights.
When multiple penalty terms are active ($\lambda_{\text{battery}}^{\text{fixed}}=0.1$, i.e., bottom row in Fig. \ref{fig:gridworld_fixed_weights}), \emph{scalarized advantages} relaxes the lava rate for $\lambda_{\text{lava}}^{\text{fixed}}\leq 0.5$, whereas \emph{scalarized rewards} behaves identically to the $\lambda_{\text{battery}}^{\text{fixed}}=0.0$ setting, effectively optimizing the lava penalty in the same way regardless of the battery penalty.

\begin{figure}[t]
    \centering
    \includegraphics[width=0.7\columnwidth]{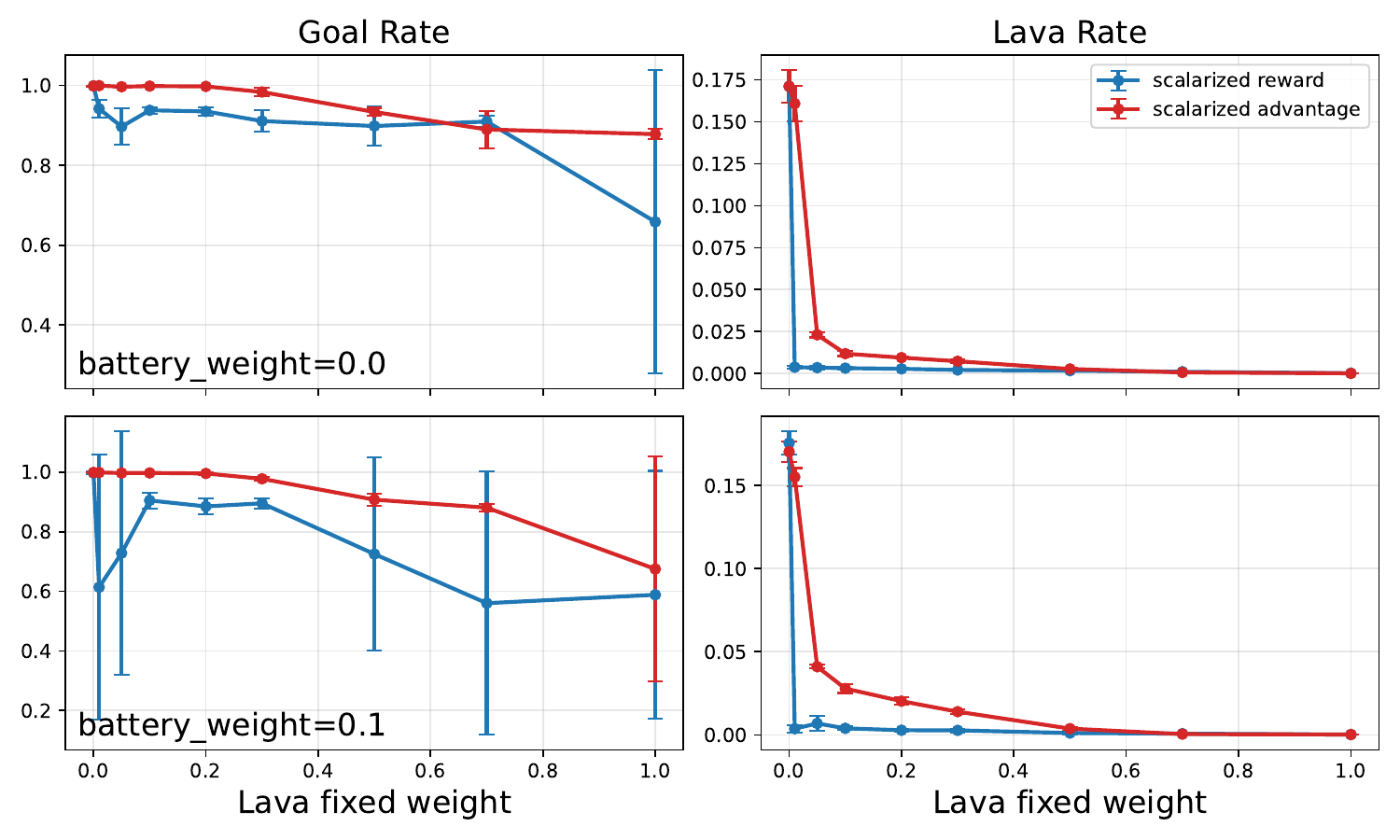} %
    \caption{Final performance of GRPO policies evaluated over 1{,}000 episodes in the gridworld with fixed penalty weights, comparing \emph{scalarized rewards} (blue) and \emph{scalarized advantages} (red). Scalarized advantages consistently achieves higher goal-reaching rates while tracking the intended trade-off induced by the lava-weight sweep more faithfully, whereas scalarized rewards suppresses lava visitation disproportionately.}

    \label{fig:gridworld_fixed_weights}
    \vspace{-2mm}
\end{figure}

Additionally, we show results for a single constraint in Figure~\ref{fig:gridworld_cgrpo_lava001_battery1}, and for two constraints (battery and lava) in Figure~\ref{fig:gridworld_cgrpo_lava001_battery01}. 
We refer the reader to the caption for more details.

\begin{figure}[t]
    \centering
    \includegraphics[width=\columnwidth]{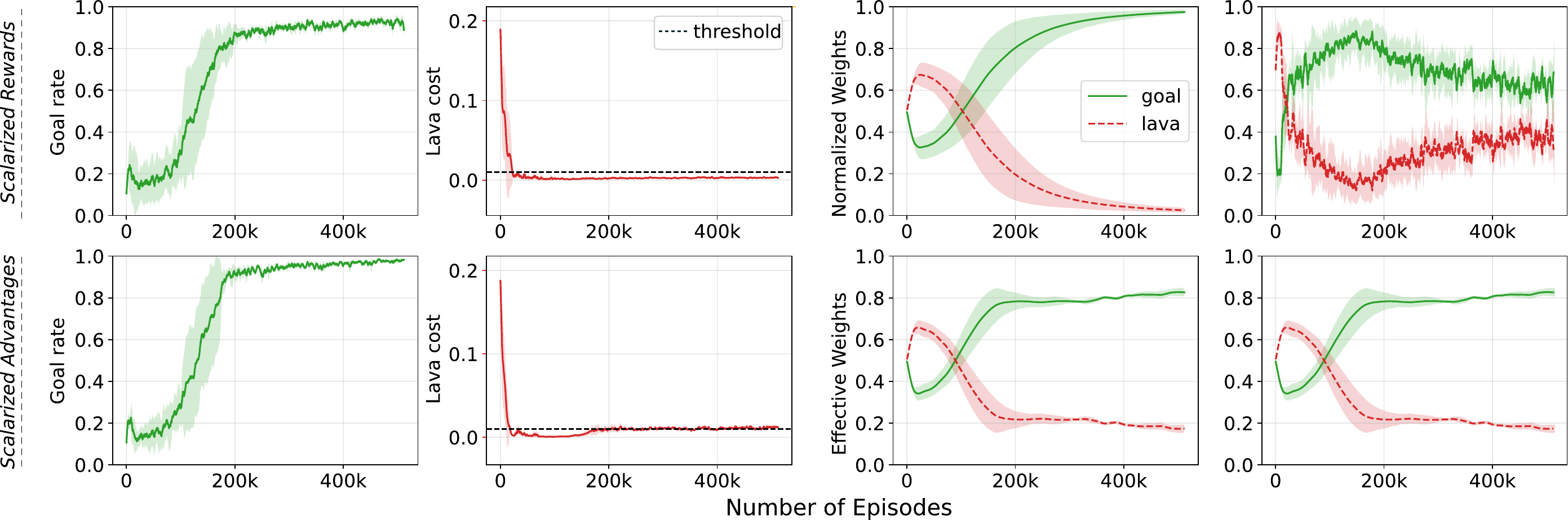}
    \caption{Learning dynamics on the gridworld under Constrained GRPO, comparing \emph{scalarized advantages} (left) and \emph{scalarized rewards} (right). In this experiment, we set $\tilde d_{\text{lava}}=0.01$ and do not enforce the battery constraint. Scalarized advantages again produces a stronger final policy and makes effective use of the allowable cost budget, with the lava rate touching the specified threshold in order to allow for success on the main task. In contrast, scalarized rewards drives the constraint rate to near-zero early in training; despite small multipliers values, the corresponding effective weights remain larger and noisier, leading to less stable learning and lower final performance.}
    \label{fig:gridworld_cgrpo_lava001_battery1}
\end{figure}

\begin{figure}[t]
    \centering
    \includegraphics[width=\columnwidth]{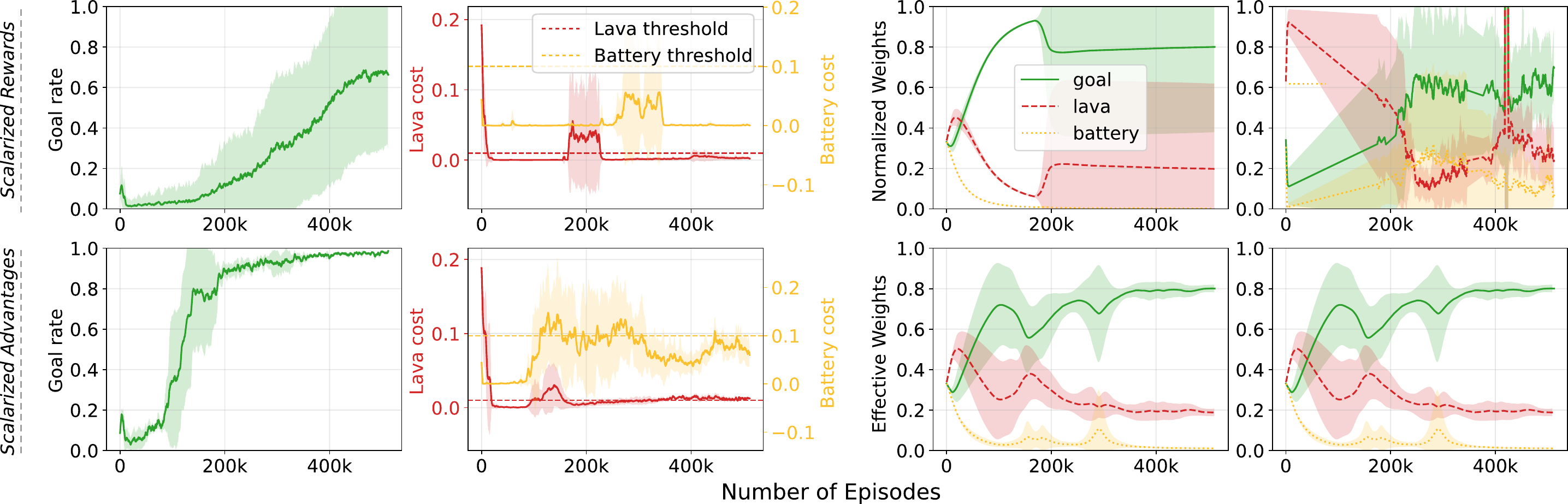}
    \caption{Learning dynamics on the gridworld under Constrained GRPO, comparing \emph{scalarized advantages} (left) and \emph{scalarized rewards} (right). In this experiment, we set $\tilde d_{\text{lava}}=0.01$ and $\tilde d_{\text{battery}}=0.1$. Scalarized advantages again produces a stronger final policy and makes effective use of its allowable costs' budget, with both the lava rate and battery rate stabilizing near their specified thresholds in order to allow for success on the main task. 
    In contrast, scalarizing rewards drives both constraints' cost rate to near-zero early in training. Even with small values of multipliers, the corresponding effective weights remain larger and noisier, leading to less stable learning and lower final performance.
    Interestingly, we see that when it tries to creep higher, the effective multipliers react excessively, resulting in the agent not adequately navigating the environment, as shown by the poor goal success rate.}
    \label{fig:gridworld_cgrpo_lava001_battery01}
\end{figure}

\section{NAVSIM v2}
\label{sec:apdx:navsim}

\subsection{Benchmark Details}

In Stage~1, the benchmark evaluates closed-loop planning quality in an open-loop manner. Concretely, each planner outputs a $4$-second trajectory sampled at $10$\,Hz, which is then executed forward without feedback using a kinematic bicycle model with an LQR controller.
Unlike NAVSIM-v1~\cite{Dauner2024NEURIPS}, which simply replayed recorded agent trajectories, NAVSIM-v2 introduces reactive behavior from other vehicles through the Intelligent Driver Model (IDM)~\cite{treiber2000idm}, improving simulation realism. Stage~2 evaluates robustness under compounding error by synthesizing new driving rollouts using Multi-Traversal Gaussian Splatting~\cite{li2025mtgs}. These scenarios are initialized from perturbed ego states (e.g., suboptimal position and/or heading), requiring the planner to produce corrective behavior.
The resulting Navhard evaluation split contains $450$ observations for Stage~1 and $5{,}462$ observations for Stage~2.

The following is the expression for EPDMS:
\begin{align}
\label{eq:apdx:epdms}
\mathrm{EPDMS} =
\underbrace{\prod_{m \in \mathcal{M}_{\mathrm{pen}}}
\mathrm{filter}_{m}(\text{agent},\,\text{human})}_{\text{penalty terms}}
\;\cdot\;
\underbrace{\left(
\frac{\sum_{m \in \mathcal{M}_{\mathrm{avg}}} w_m \, \mathrm{filter}_{m}(\text{agent},\,\text{human})}
     {\sum_{m \in \mathcal{M}_{\mathrm{avg}}} w_m}
\right)}_{\text{weighted average terms}}
\,.
\end{align}
To improve fairness, the function $\mathrm{filter}_{m}$ ignores penalty violations (i.e., returns $1$ instead of $0$) whenever the human expert driver makes the same violation. The weights $w_m$ for the weighted average terms are: (1) $w_m=5$ for EP, (2) $w_m=5$ for TTC, (3) $w_m=2$ for HC, (4) $w_m=2$ for EC, and (5) $w_m=2$ for LK.

\subsection{Base Model Pretraining}

We initialize all experiments from the Qwen2.5-VL 3B Instruct vision-language model~\cite{Bai2025Qwen25VLTR}. We first perform supervised finetuning on the CoVLA dataset~\cite{Arai2024CoVLACV}, which contains $83$ hours of nominal driving collected in Japan. Following Rowe et al.~\cite{rowe2025poutine}, we represent planned trajectories as serialized text and train the model using next-token prediction in the original vocabulary.

To promote richer planning behavior, we additionally include automatically generated chain-of-thought annotations, produced in a zero-shot manner by Qwen2.5-VL 72B Instruct, and prepend these to the target trajectory sequences during training. After this initialization stage, we further supervise finetune the model on $92{,}000$ scenes from the NAVSIM training set~\cite{Cao2025CORL, Dauner2024NEURIPS}. The remaining $\approx 11{,}000$ scenes are held out and used exclusively for reinforcement learning finetuning (RLFT) in our benchmark experiments.

\subsection{Hyperparameters}

We finetune all our models for two epochs through the $11$k held-out scenes for RL finetuning (RLFT).
All our experiments are performed on 8 A100 GPUs, with each RLFT run taking nearly 20 hours.
We set $\beta = 0.02$ to constrain the policy to the reference SFT policy, the learning rate to $1e^{-6}$, and the group size to $16$. We perform two iterations of policy optimization per batch of data.
We update the Lagrangian multipliers at the same rate as the policy.

\subsection{PPO and PPO Lagrangian Details}

The value network is a separate Qwen2.5-VL 3B model initialized from the same SFT checkpoint as the policy. This backbone is finetuned using LoRA \cite{hu2022lora} and we attach separate linear value heads for each reward and constraint signal to the last hidden state of the Qwen model. 
Advantages are computed per-token using GAE ($\lambda=0.95$, $\gamma=0.99$), with the value network trained for 200 warmup steps before policy optimization begins.

\subsection{Additional Results}

In Figure~\ref{fig:train_constrained_multipliers_grpo}, we show representative learning dynamics of the multipliers and corresponding constraints, highlighting differences between the two advantage estimation approaches for constrained GRPO in NAVSIM.

\begin{figure*}[h!t]
    \centering
    \includegraphics[width=\textwidth]{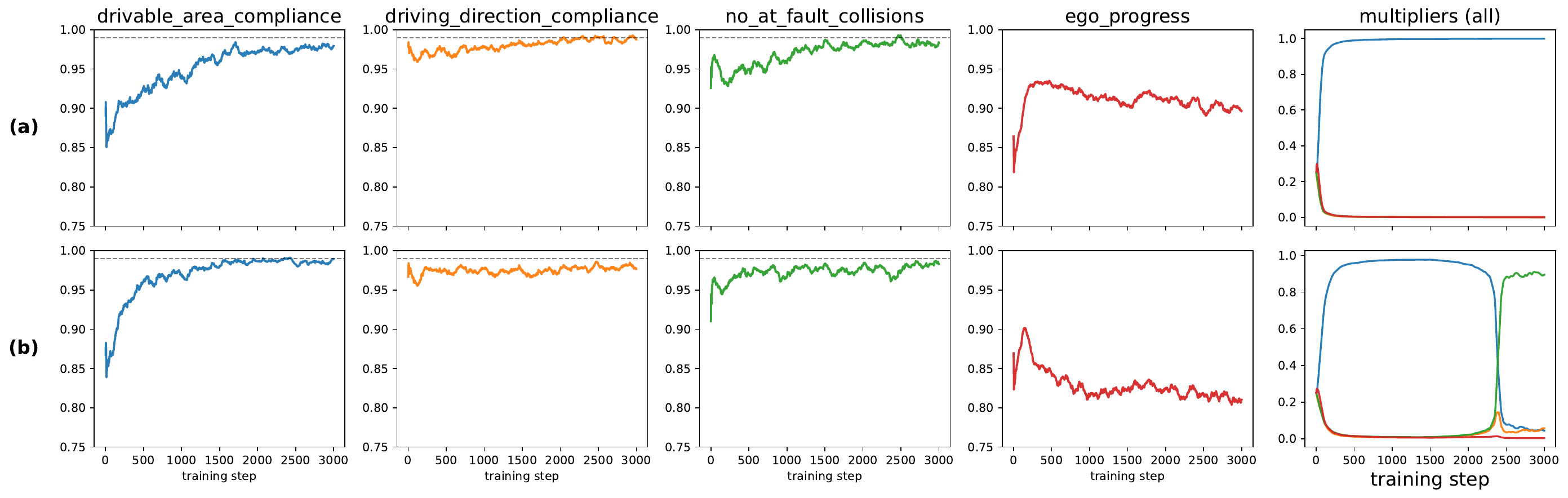}
    \caption{Constrained GRPO plots of constraint and multiplier values through training, comparing with \textbf{(a)} scalarized rewards with \textbf{(b)} scalarized advantages. Focusing on the ego progress and the drivable area compliance, we see that with the scalarized rewards, the policy does not adhere to the multiplier settings as we intend. First, the drivable area compliance (DAC) stays below its minimum threshold $0.99$ throughout its training even though the multiplier value had moved close to its maximum value of $1.0$. Moreover, we see that even though the ego-progress (a component of the WAT reward) is set to a very low value, the scalarized rewards variant still maximizes it. Conversely, the scalarized advantage model allows the different components to adhere to the multiplier setting. In addition, we see that at around $2500$ training steps, something interesting happens in the scalarized advantage model's multipliers; since the DAC constraint is satisfied and the no at-fault collision (NC) takes a sharp dip, the Lagrangian optimization pushes the NC multiplier higher and squashes the DAC multiplier. This confirms our theoretical findings that scalarizing the advantages leads to better adherence to the specified multiplier weights.}
    \label{fig:train_constrained_multipliers_grpo}
\end{figure*}

\section{GSM8K}

\subsection{Hyperparameters}
We employ huggingface pretrained Qwen2.5-Instruct models with varying sizes.
All models are RL-finetuned on a single node of 8x H200 GPUs, taking under three hours to complete four epochs of training, which was sufficient. 
We use the AdamW optimizer with a learning rate of $1e-6$ for 3B and 7B models, and a learning rate of $5e-6$ for 1.5B models; the Lagrangian parameters are optimized with the Adam optimizer with a learning rate of $1e-4$.

The constraint violation thresholds are different for different model sizes, since each has its own capability and setting thresholds too high can result in infeasible policies.
For all model size, the format adherence and integer validity thresholds were set at $0.99$; so the policy needed to respect the format and integer constraint $99\%$ of the time.
For the correctness constraint, the 1.5B sized model all employed a $75\%$; the 3B employed a $90\%$ threshold; and the 7B employed a $95\%$. 
On the training set, these thresholds were all respected; however, as can be seen in Table~\ref{tab:gsm8k_main}, there is a small shift in the performance across the board (ppo-lagrangian and constrained GRPO).

\section{Limitations}
\label{sec:apdx:limitations}

In this section, we describe some of the limitations of constrained GRPO.
In constrained policy optimization, a general problem is how one should choose the thresholds of the different behavioural constraints. 
While in this work we mitigate this problem by using interpretable indicator cost functions, it is unclear how to effectively choose these thresholds such that the constraint remain feasible.
In addition, with multiple constraints interacting and competing together, progress on certain constraints could stall completely when one of the constraints is infeasible, effectively killing the optimization process.
An interesting direction for future work is to couple Constrained GRPO with curriculum-style mechanisms that adapt constraint thresholds over training, potentially improving exploration and enabling satisfaction of stricter constraints.

\end{document}